\documentclass[letterpaper]{article}
\usepackage{aaai}
\usepackage{times}
\usepackage{helvet}
\usepackage{courier}

\usepackage{amsmath,amsfonts,amssymb,amsthm}

\usepackage[usenames,dvipsnames]{color}
\usepackage{tikz}
\usetikzlibrary{trees}
\usepackage{tikz-qtree}
\usepackage{subfig}
\usepackage{url}
\usepackage{algorithm}
\usepackage[noend]{algorithmic}
\renewcommand{\algorithmiccomment}[1]{\hfill{// #1}}
\makeatletter
\newcounter{ALC@tempcntr}
\newcommand{\LCOMMENT}[1]{%
    \setcounter{ALC@tempcntr}{\arabic{ALC@rem}}
    \setcounter{ALC@rem}{1}
    \item \algorithmiccomment{#1}
    \setcounter{ALC@rem}{\arabic{ALC@tempcntr}}
}%
\makeatother

\newcommand*\rot{\rotatebox{90}}

\newcommand{\citep}[1]{\cite{#1}}
\newcommand{\citet}[1]{\citeauthor{#1}~\shortcite{#1}}

\theoremstyle{plain}
\newtheorem{theorem}{Theorem}
\newtheorem{lemma}[theorem]{Lemma}
\newtheorem{corollary}[theorem]{Corollary}

\theoremstyle{definition}

\def\xypart{($\X$,$\Y$)-partition}

\def\co{{\sf \bf CO}}
\def\va{{\sf \bf VA}}
\def\ce{{\sf \bf CE}}
\def\im{{\sf \bf IM}}
\def\eq{{\sf \bf EQ}}
\def\ct{{\sf \bf CT}}
\def\se{{\sf \bf SE}}
\def\me{{\sf \bf ME}}

\newcommand\shrink[1]{}

\newcommand\underbracetxt[2]{\underbrace{{#1}}_{\text{{#2}}}}

\def\bpair(#1,#2){{
\setlength{\tabcolsep}{2.5pt}
\renewcommand{\arraystretch}{0.8}
\begin{tabular}{|c|c|}\hline\(#1\)&\(#2\)\\\hline\end{tabular}}}
\def\F(#1){{\langle#1\rangle}}

\def\apply{\texttt{Apply}}
\def\set{{\leftarrow}}

\def\cache{{\tt Cache}}
\def\uniqueD{{\tt UniqueD}}

\def\compress{{\tt Compress}}
\def\vtreelca{{\tt VtreeLCA}}
\def\partition{{\tt Partition}}
\def\nil{{\tt nil}}
\def\bop{\circ}

\def\x{\mathbf{x}}
\def\X{\mathbf{X}}
\def\y{\mathbf{y}}
\def\Y{\mathbf{Y}}

\def\Z{\mathbf{Z}}

\newcommand\name[1]{\ensuremath{\mathsf{#1}}}
\def\true{\name{\top}}
\def\false{\name{\bot}}

\def\c2d{{\sc c2d}}

\DeclareMathOperator{\poly}{poly}

\tikzstyle{vtree}=[
    level distance=0.9cm,
    sibling distance=0.15cm
]

\def\ok(#1){{\it ok}{#1}}

\def\cd{{\sf \bf CD}}
\def\fo{{\sf \bf FO}}
\def\ac{{\sf \bf \wedge C}}
\def\oc{{\sf \bf \vee C}}
\def\nc{{\sf \bf \neg C}}

\def\sfo{{\sf \bf SFO}}
\def\bac{{\sf \bf \wedge BC}}
\def\boc{{\sf \bf \vee BC}}

\def\y{$\surd$}

\def\bn{$\bullet$}

\newcommand{\mytitle}{On the Role of Canonicity in Bottom-up Knowledge Compilation} 

\nocopyright

\begin{document}
%
\title{\mytitle}
\author{
Guy Van den Broeck \and 
Adnan Darwiche\\
Computer Science Department
\\University of California, Los Angeles\\ \texttt{\{guyvdb,darwiche\}@cs.ucla.edu}
}
\maketitle
\begin{abstract}
\begin{quote}
We consider the problem of bottom-up compilation of knowledge bases, which is usually predicated on the
existence of a polytime function for combining compilations using Boolean operators (usually called an \apply\ function). While such a polytime \apply\ function is known
to exist for certain languages (e.g., OBDDs) and not exist for others (e.g., DNNF), its existence for certain languages remains unknown. Among the latter
is the recently introduced language of Sentential Decision Diagrams (SDDs), for which a polytime \apply\ function exists for unreduced SDDs, but remains unknown for
reduced ones (i.e. canonical SDDs). We resolve this open question in this paper and consider some of its theoretical and practical implications. Some of the findings we 
report question the common wisdom on the relationship between bottom-up compilation, language canonicity and the complexity of the \apply\ function. 
\end{quote}
\end{abstract}

\section{Introduction} 

The Sentential Decision Diagram (SDD) is a recently proposed circuit representation of propositional knowledge bases~\citep{Darwiche11}.
The SDD is a target language for knowledge compilation~\citep{selman1996knowledge,darwicheJAIR02}, 
meaning that once a propositional knowledge base is compiled into an SDD, the SDD can be reused to answer multiple hard queries efficiently (e.g., clausal entailment or model counting). 

SDDs subsume Ordered Binary Decision Diagrams (OBDD)~\citep{Bryant86} and come with a tighter size bound~\citep{Darwiche11}, 
while still being equally powerful as far as their polytime support for classical queries (e.g., the ones in~\citet{darwicheJAIR02}).
Moreover, SDDs are subsumed by d-DNNFs~\citep{darwiche2001tractability}, which received much attention over the last decade, for fault diagnosis~\citep{elliot:06}, planning~\citep{palacios:05}, databases~\citep{suciu2011probabilistic}, but most importantly for probabilistic inference~\cite{chavira2006compiling,Chavira.Darwiche.AIJ.2008,FierensBTGR11}.
Even though SDDs are less succinct than d-DNNFs, they can be compiled {\em bottom-up,} just like OBDDs. For example, a clause can be compiled by disjoining the SDDs corresponding
to its literals, and a CNF can be compiled by conjoining the SDDs corresponding to its clauses. This bottom-up compilation is implemented using the \apply{} function, which combines
two SDDs using Boolean operators.\footnote{The \apply\ function (and its name) originated in the OBDD literature~\citep{Bryant86}}
Bottom-up compilation makes SDDs attractive for certain applications (e.g., probabilistic inference~\citep{ChoiKisaDarwiche13})
and can be critical when the knowledge base to be compiled is constructed incrementally (see the discussion in~\citet{pipatsrisawat08compilation}). 

According to common wisdom, a language supports bottom-up compilation only if it supports a polytime \apply\ function. For example, OBDDs are known to support bottom-up compilation 
and have traditionally been compiled this way. In fact, the discovery of SDDs was mostly driven by the need for bottom-up compilation, which was preceded by the discovery of 
{\em structured decomposability}~\cite{pipatsrisawat08compilation}:
a property that enables some Boolean operations to be applied in polytime. SDDs satisfy this property and stronger ones, leading to a polytime \apply\ function~\citep{Darwiche11}.
This function, however, assumes that the SDDs are unreduced (i.e., not canonical). For reduced SDDs, the 
existence of a polytime \apply\ function has been an open question since SDDs were first introduced (note, however, that both reduced and unreduced OBDDs are 
supported by a polytime \apply\ function).

We resolve this open question in this paper, showing that such an \apply\ function does not exist in general. We also pursue some theoretical and practical implications of this result,
on bottom-up compilation in particular. On the practical side, we reveal an empirical finding that seems quite surprising: bottom-up compilation with reduced SDDs is much more feasible
practically than with unreduced ones, even though the latter supports a polytime \apply\ function while the former does not. This finding
questions common convictions on the relative importance of a polytime \apply{} in contrast to canonicity as desirable properties for a language that supports efficient bottom-up compilation.
 On the theoretical side, we show that certain 
transformations (e.g., conditioning) can lead to blowing up the size of reduced SDDs, while they don't for unreduced SDDs. Finally, we identify a subset of SDDs for which
a polytime \apply\ exists even under reduction.

\shrink{
In the same way that OBDDs are characterized by total variable orders, SDDs are characterized by vtrees (a vtree is an ordered, binary tree whose leaves are labelled with variables).
It is known that for a given Boolean function and a total variable order, there is a unique {\em reduced} OBDD~\citep{Bryant86}. 
The same is true for SDDs: for a given Boolean function and a vtree, there is a unique {\em reduced} SDD which is based on a property known as {\em compression}~\citep{Darwiche11}. 
If compression is not enforced, \citet{Darwiche11} showed that the \apply{} function, which allows one to conjoin and disjoin SDDs, can be implemented in polynomial time.
On compressed SDDs, however, the complexity of \apply{} and other transformations were not known, leaving this as an open question.
}

\section{Background} 

We will use the following notation for propositional logic.
Upper-case letters (e.g., $X$) denote propositional \emph{variables} and lower-case letters denote their \emph{instantiations} (e.g., $x$).
Bold letters represent sets of variables (e.g., $\X$) and their instantiations (e.g., $\x$).
A \emph{literal} is a variable or its negation.
A \emph{Boolean function} $f(\X)$ maps each instantiation $\x$ to $\true$ (true) or $\false$ (false).

\subsection{The SDD Representation}

\begin{figure}[tb]
  \centering
  \subfloat[An SDD]{
    \label{fig:sdd b}
    \includegraphics[width=0.29\textwidth,clip=true,angle=0]{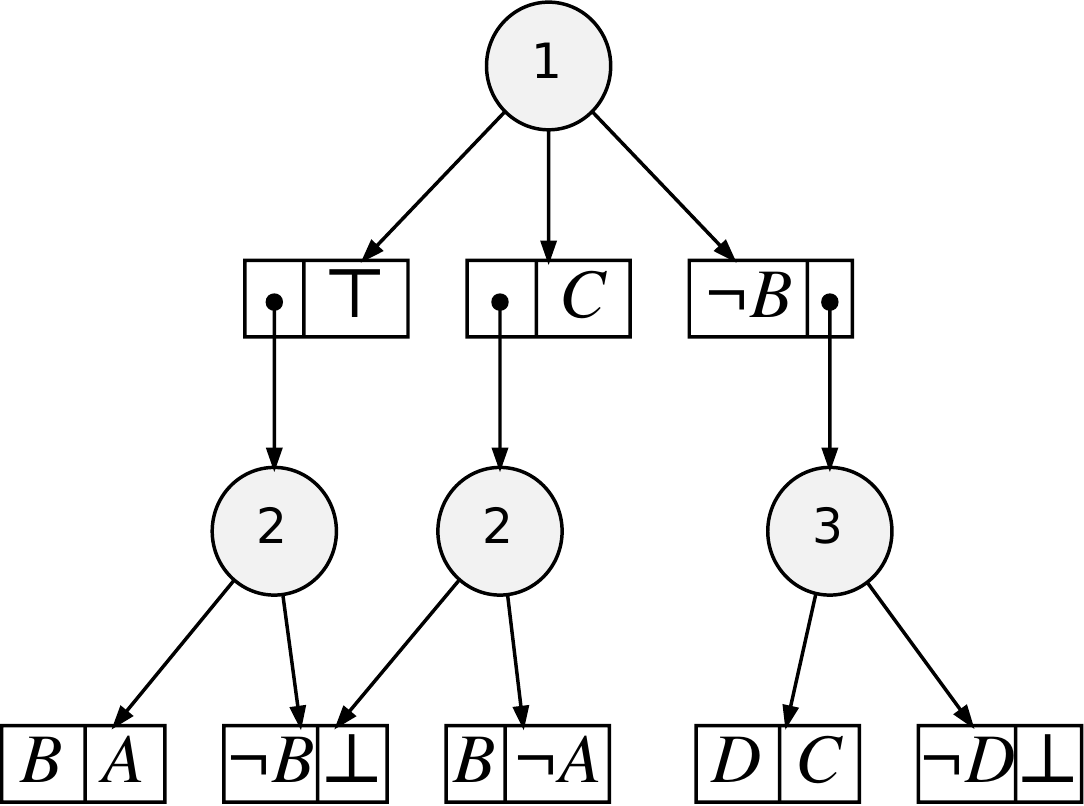}
  }~
  \subfloat[A vtree]{
    \label{fig:vt b}
    \begin{tikzpicture}[vtree]
      \node{\Tree 
        [.$1$ 
          [.$2$ 
            $B$ 
            $A$ 
          ] 
          [.$3$
            $D$ 
            $C$ 
          ] 
        ]};
    \end{tikzpicture}
  }
  \caption{An SDD and vtree for \((A\wedge B) \vee (B \wedge C) \vee (C \wedge D)\).}
  \label{fig:sdd notation}
\end{figure}
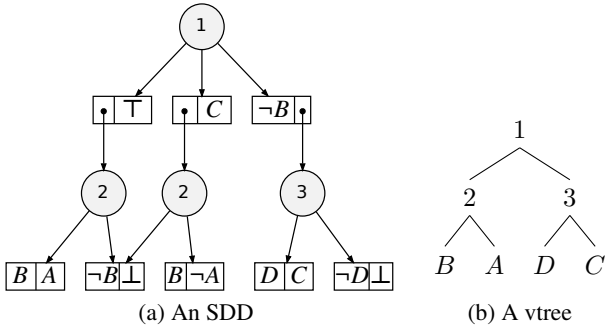

The Sentential Decision Diagram (SDD) is a newly introduced representation language for propositional knowledge bases~\cite{Darwiche11}. 
Figure~\ref{fig:sdd b} depicts an SDD: paired-boxes \(\bpair(p,s)\) are called {\em elements} and represent conjunctions (\(p \wedge s\)), where \(p\) is called a {\em prime} and \(s\) is called a {\em sub.} Circles are called {\em decision nodes} and represent disjunctions of their child elements.  

An SDD is constructed for a given {\em vtree,} which is a full binary tree whose leaves are variables; see for example Figure~\ref{fig:vt b}.
Every node in an SDD respects some vtree node (except for \(\top\) and \(\bot\)).
SDD literals respect the vtree leaf labeled with their variable. In Figure~\ref{fig:sdd b}, decision nodes are labeled with the vtree node they respect. 
Consider a decision node with elements \(\bpair(p_1,s_1), \ldots, \bpair(p_n,s_n),\) and suppose that it respects a vtree node \(v\) which has variables \(\X\) in its left subtree and variables \(\Y\) in its right subtree. 
We are then guaranteed that each prime \(p_i\) will only mention variables in \(\X\) and that each sub \(s_i\) will only mention variables in \(\Y\).  
Moreover, the primes are guaranteed to represent propositional sentences that are consistent, mutually exclusive, and exhaustive.
This type of decomposition is called an \xypart{}~\cite{Darwiche11}.
For example, the top decision node in Figure~\ref{fig:sdd b} has the following elements
\begin{align} \label{f:compressedxypart}
\{ (\underbracetxt{A \wedge B}{prime},
    \underbracetxt{\true}{sub}), 
   (\underbracetxt{\neg A \wedge B}{prime},
    \underbracetxt{C}{sub}), 
   (\underbracetxt{\neg B}{prime},
    \underbracetxt{D \wedge C}{sub}) \},
\end{align}
which correspond to an $(AB,CD)$-partition of the function \((A\wedge B) \vee (B \wedge C) \vee (C \wedge D)\).
One can verify that the primes and subs of this partition satisfy the properties mentioned above.

An \xypart{} is \emph{compressed} if it has distinct subs, and an SDD is compressed if its \xypart{}s are compressed.
A Boolean function may have multiple \xypart{}s, but the compressed partition is unique. 
Our example function has another $(AB,CD)$-partition, which is not compressed:
\begin{align} \label{f:xypart}
\{& (A \wedge B,
    \true), 
   (\neg A \wedge B,
    C), \nonumber\\
  & (A \wedge \neg B,
    D \wedge C), 
   (\neg A \wedge \neg B,
    D \wedge C) \}.
\end{align}
An uncompressed \xypart{} can be compressed by merging all elements $(p_1,s),\dots,(p_n,s)$ that share the same sub into one element $(p_1 \lor \dots \lor p_n,s)$.
Compressing (\ref{f:xypart}) combines the two last elements into $([A \wedge \neg B] \lor [\neg A \wedge \neg B],D \wedge C) = (\neg B,D \wedge C)$, resulting in~(\ref{f:compressedxypart}). This is the unique compressed $(AB,CD)$-partition.

Given a vtree, each Boolean function also has a unique \emph{compressed} SDD, when this property is combined with either {\em trimming} or {\em normalization}  properties~\cite{Darwiche11}.
These are weaker properties that mildly affect the size of an SDD. For example, 
a trimmed SDD contains no decision nodes of the form $\{(\top,\alpha)\}$ or $\{(\alpha,\top),(\neg \alpha, \bot)\}$ (we can trim an SDD by replacing these nodes with $\alpha$). 
Compressed and trimmed SDDs are canonical, and so are compressed and normalized SDDs~~\cite{Darwiche11}.

OBDDs correspond precisely to SDDs that are constructed using a special type of vtree, called a right-linear
vtree~\cite{Darwiche11}. The left child of each inner node in these vtrees is a variable.
With right-linear vtrees, compressed/trimmed SDDs correspond to reduced OBDDs, while compressed/normalized SDDs correspond to 
oblivious OBDDs~\cite{xue12} (reduced and oblivious OBDDs are also canonical).
We will refer to compressed and trimmed SDDs as {\em reduced} SDDs and restrict our attention to them in the rest of the paper.

The size of an OBDD depends critically on the underlying variable order. Similarly, the size of an SDD depends critically on the vtree used (right-linear vtrees correspond to variable orders).
Vtree search algorithms can sometimes find SDDs that are orders-of-magnitude more succinct than OBDDs found by searching for variable orders~\citep{ChoiDarwiche13}.  Such algorithms 
assume canonical SDDs, allowing one to search the space of SDDs by searching the space of vtrees.

\shrink{
The set of all SDDs will be denoted by \SDD{}, and the subset of canonical SDDs by \cSDD\ (i.e., compressed and trimmed). For a given Boolean function, these languages may contain
multiple SDDs that respect different vtrees. Suppose now that we have a function \(T\), called a {\em vtree function,} which maps each integer \(n\) into a specific vtree. Suppose further that the SDD for a function \(f(X_1,\ldots,X_n)\) respects the vtree \(T(n)\).\footnote{We are assuming that the function \(f(X_1,\ldots,X_n)\) essentially depends on all its variables.} 
This induces the subsets  \SDD$_T$ and \cSDD$_T$ of  \SDD\ and  \cSDD.
}

\subsection{Queries and Transformations}

\begin{table}
\centering
\small
\begin{tabular}{|c|c|c|c|c|} \hline
Query & Description & OBDD & SDD & d-DNNF \\ \hline \hline
\co & consistency & $\surd$ & $\surd$ & $\surd$ \\ \hline
\va & validity & $\surd$ & $\surd$ & $\surd$ \\  \hline
\ce &clausal entailment &  $\surd$ & $\surd$ & $\surd$ \\  \hline
\im &implicant check &  $\surd$ & $\surd$ & $\surd$ \\  \hline
\eq & equivalence check & $\surd$ & $\surd$ & ? \\  \hline
\ct &model counting &  $\surd$ & $\surd$ & $\surd$ \\  \hline
\se & sentential entailment &  $\surd$ & $\surd$ & \(\circ\) \\  \hline
\me & model enumeration & $\surd$ & $\surd$ & $\surd$ \\ \hline
\end{tabular}
\caption{$\surd$ means that a polytime algorithm exists for the corresponding language/query, while \(\circ\) means that no such algorithm
exists unless \(P=NP\). \label{tab:queries}}
\end{table}

SDDs are a strict subset of deterministic, decomposable negation normal form (d-DNNF). 
They are actually a strict subset of structured d-DNNF and, hence, support the same polytime queries supported by structured d-DNNF~\cite{pipatsrisawat08compilation};
see Table~\ref{tab:queries}. This makes SDDs as powerful as OBDDs in terms of their support for these queries.

\begin{algorithm}[tb]
\caption{\(\apply(\alpha,\beta,\bop)\)
\label{alg:apply}}
\begin{algorithmic}[1]
\IF{\(\alpha\) and \(\beta\) are constants or literals}
\RETURN \(\alpha \bop \beta\) \COMMENT{\emph{result is a constant or literal}}
\ELSIF{\(\cache(\alpha,\beta,\bop) \neq \nil\)} 
\RETURN \(\cache(\alpha,\beta,\bop)\) \label{ln:base end} \label{ln:cache} \COMMENT{\emph{has been computed before}}
\ELSE
\STATE \(\gamma \set \{\}\)  \label{ln:recurse start}
\FORALL{elements \((p_i,s_i)\) in $\alpha$} \label{ln:elem1}
\FORALL{elements \((q_j,r_j)\) in $\beta$} \label{ln:elem2}
\STATE \(p \set \apply(p_i,q_j,\wedge)\)  \label{ln:consistent primes}
\IF{\(p\) is consistent} \label{ln:consistency}
\STATE \(s \set \apply(s_i,r_j,\bop)\)
\STATE add element \((p,s)\) to \(\gamma\) \label{ln:distinct subs}
\ENDIF
\ENDFOR
\ENDFOR
\STATE (optionally) $\gamma \gets \compress(\gamma)$ \label{ln:compress} \COMMENT{\emph{compression}}
\LCOMMENT{\emph{get unique decision node and return it}} 
\RETURN \(\cache(\alpha,\beta,\bop) \set \uniqueD(\gamma)\) \label{ln:recurse end}
\ENDIF
\end{algorithmic}
\end{algorithm}

SDDs satisfy stronger properties than structured d-DNNFs, allowing one, for example, to conjoin or disjoin two SDDs in polytime.
Algorithm~\ref{alg:apply} shows the outline of an \apply{} function~\citep{Darwiche11} that takes two SDDs $\alpha$ and $\beta$, and a binary Boolean operator $\circ$ (e.g., $\land$, $\lor$, xor), and returns the SDD for $\alpha \circ \beta$. This code assumes that the SDD is normalized instead of trimmed. The code for trimmed SDDs is similar, although a bit more detailed.
Line~\ref{ln:compress} optionally compresses each partition, in order to return a compressed SDD. Without compression, this algorithm has a time and space complexity of \(O(n m)\),
where \(n\) and \(m\) are the sizes of input SDDs~\cite{Darwiche11}. This comes at the expense of losing canonicity. Whether a polytime complexity can be attained under compression is an open question.

There are several implications of this question. For example, depending on the answer, one would know whether certain {\em transformations,} such as conditioning and existential quantification,
can be supported in polytime on reduced SDDs. Moreover, according to common wisdom, a negative answer may preclude bottom-up compilation from being feasible on reduced SDDs. We explore this question and its implications next.

\shrink{
The $\vtreelca(\alpha,\beta)$ function returns the least common ancestor of the vtree nodes respected by $\alpha$ and $\beta$, or if one of them is a constant, the vtree node respected by the other.
The $\partition(\alpha,T)$ function returns the unique compressed partition that corresponds to $\alpha$ and that respects $T$.\footnote{If $\alpha$ is $\top$ ($\bot$), 
then $\partition(\alpha,T)$ returns \ $\{(\top,\top)\}$ ($\{(\top,\bot)\}$).
If $\alpha$ respects $T$, then $\partition(\alpha,T)$ returns the partition of the root decision node of \(\alpha\). 
If $\alpha$ respects a node in the left (right) subtree of $T$, then $\partition(\alpha,T)$ returns $\{(\alpha,\top),(\neg \alpha,\bot)\}$ ($\{(\top, \alpha)\}$).}
}

\section{Complexity of \apply\ on Reduced SDDs} 

The size of a decision node is the number of its elements, and the size of an SDD is the sum of sizes attained by its decision nodes. 
We now show that reduction, given a fixed vtree, may blow up the size of an SDD.

\shrink{
\begin{theorem}
\SDDs{} are strictly more succinct than \cSDDs{}, that is, there does not exist a polynomial $\poly$ for which every sentence $\alpha \in \SDD$ has an equivalent sentence $\beta \in \cSDD$ such that $|\beta| \leq \poly(|\alpha|)$.
\end{theorem}
\begin{corollary} \label{thm:compress}
  Suppose that we are given a vtree $T$ and an \SDD{}~$\alpha$ that respects $T$.
  The complexity of obtaining the unique \cSDD{}~$\beta$ that respects $T$, and that is equivalent to $\alpha$, can grow worst-case exponentially in the size of~$\alpha$.
\end{corollary}
}
\begin{theorem} \label{theo:compression}
There exists a class of Boolean functions \(f_m(X_1,\ldots,X_m)\) and corresponding vtrees \(T_m\) such that \(f_m\) has an SDD of size \(O(m^2)\) wrt vtree \(T_m\), yet
the reduced SDD of function \(f_m\) wrt vtree \(T_m\) has size \(\Omega(2^m)\).
\end{theorem}
\begin{proof}
  Consider the function 
  $f^a_n(\X,\Y,Z) = \bigvee_{i=1}^n  \left(\bigwedge_{j=1}^{i-1} \neg Y_j\right) \land Y_i \land X_i$ which has \(2n+1\) variables. Of these, $Z$ is non-essential.
  Consider a vtree $T_n$ of the form 
  \begin{center}
    \begin{tikzpicture}[vtree]
      \node{\Tree [.1 [.2 $\X$ $\Y$ ] $Z$ ]};
    \end{tikzpicture}    
  \end{center}
  where the sub-vtrees over variables $\X$ and $\Y$ are arbitrary.
  We will now construct an uncompressed SDD for this function using vtree $T_n$ and whose size is \(O(n^2)\).
  We will then show that the compressed SDD for this function and vtree has a size \(\Omega(2^n)\).
  
  The first step is to construct a partition of FUNCTION $f^a_n$ that respects the root vtree node, that is, an (\textbf{XY},Z)-partition. Consider
  \begin{align*}
  \left\{ 
    \begin{array}{l}
      (Y_1 \land X_1,\true), \\
      (\neg Y_1 \land Y_2 \land X_2,\true),\\
      \dots,\\
      (\neg Y_1 \land \dots \land \neg Y_{n-1} \land Y_n \land X_n,\true),\\
      (Y_1 \land \neg X_1,\false), \\
      (\neg Y_1 \land Y_2 \land \neg X_2,\false),\\
      \dots,\\
      (\neg Y_1 \land \dots \land \neg Y_{n-1} \land Y_n \land \neg X_n,\false),\\
      (\neg Y_1 \land \dots \land \neg Y_n,\false)
    \end{array} 
  \right\},
  \end{align*}
  which is equivalently written as
  \begin{align*}
  \bigcup_{i=1}^n & \left\{ \left(  \bigwedge_{j=1}^{i-1} \neg Y_j  \land Y_i \land X_i,\true\right),\right.\\
  \textstyle
  & ~~~ \left.\left(  \bigwedge_{j=1}^{i-1} \neg Y_j  \land Y_i \land \neg X_i,\false \right) \right\}
  \cup \left\{ \left( \bigwedge_{j=1}^{n} \neg Y_j,\false \right)\right\}.
  \end{align*}
  The size of this partition is $2n+1$, and hence linear in $n$. It is \emph{uncompressed}, because there are $n$ elements that share sub $\true$ and $n+1$ elements that share sub $\false$.
  The subs already respect the leaf vtree node labelled with variable~$Z$.
  
  In a second step, each of prime above is written as a compressed \xypart{} that respects the left child of the vtree root. Prime \(\bigwedge_{j=1}^{i-1} \neg Y_j \land Y_i \land X_i\) becomes 
  \begin{align*}
    \left\{ \left(\phantom{\neg} X_i , \bigwedge_{j=1}^{i-1} \neg Y_j  \land Y_i \right), \left(\neg X_i,\false\right)\right\},  
  \end{align*}
  prime \( \bigwedge_{j=1}^{i-1} \neg Y_j  \land Y_i \land \neg X_i \) becomes
  \begin{align*}
  \left\{ \left(\neg X_i,\bigwedge_{j=1}^{i-1} \neg Y_j  \land Y_i \right), \left(\phantom{\neg} X_i, \false\right) \right\}
  \end{align*}
  and prime \(\bigwedge_{j=1}^{n} \neg Y_j\) becomes
  \begin{align*}
   \left\{ \left(\true, \bigwedge_{j=1}^{n} \neg Y_j \right)\right\}.
  \end{align*}
  The sizes of these partitions are bounded by \(2\).
  
  Finally, we need to represent the above primes as SDDs over variables \(\X\) and the subs as SDDs over variables \(\Y\).
  Since these primes and subs correspond to terms (i.e. conjunctions of literals), each has a compact SDD representation, independent of the chosen sub-vtree over
  variables $\X$ and $\Y$.
  For example, we can choose a right-linear vtree over variables \(\X\), and similarly for variables \(\Y\), leading to an OBDD representation of each prime and sub,
  with a size linear in $n$ for each OBDD. The full SDD for function \(f^a_n\) will then have a size which is \(O(n^2)\). Recall that this SDD is uncompressed as some of its decision nodes have elements with equal subs.
  
  The compressed SDD for this function and vtree is unique. We now show that its size must be \(\Omega(2^n)\).
  We first observe that the unique, compressed (\textbf{XY},Z)-partition of function \(f^a_n\) is 
  \begin{align*}
  &\left\{ \left( \bigvee_{i=1}^n  \left(\bigwedge_{j=1}^{i-1} \neg Y_j\right)  \land Y_i \land X_i,\true\right), \right. \\
  & ~~~~ \left.\left( \left[ \bigvee_{i=1}^n  \left(\bigwedge_{j=1}^{i-1} \neg Y_j\right)  \land Y_i \land \neg X_i \right] \lor \left[ \bigwedge_{j=1}^{n} \neg Y_j \right],\false \right) \right\}.
  \end{align*}
  Its first sub is the function 
  \begin{align*}
    f_n^b(\X,\Y) = \bigvee_{i=1}^n \left(\bigwedge_{j=1}^{i-1} \neg Y_j\right) \land Y_i \land X_i,
  \end{align*}
  which we need to represent as an \xypart{} to respect left child of the vtree root. However, \citet{xue12} proved the following.
  \begin{lemma} \label{thm:expsddx}
    The compressed \xypart{} of $f_n^b(\X,\Y)$ has $2^n$ elements.
  \end{lemma}
  This becomes clear when looking at the function $f_n^b$ after instantiating the $\X$-variables. Each distinct $\x$ results in a unique subfunction $f_n^b(\x,\Y)$, and all states $\x$ are mutually exclusive and exhaustive. Therefore, 
  $$\{(\x,f_n^b(\x,\Y))~|~ \x \text{ instantiates } \X \}$$ 
  is the unique, compressed \xypart{} of function $f_n^b(\X,\Y)$, which has $2^n$ elements. Hence, the compressed SDD must have size \(\Omega(2^n)\).
\end{proof}

Theorem~\ref{theo:compression} has a number of implications, 
which are summarized in Table~\ref{tab:transformations}; see also \citet{darwicheJAIR02}.
\begin{theorem} \label{theo:transformations}
The results in Table~\ref{tab:transformations} hold.
\end{theorem}
The proof of this theorem is in the supplementary material.
First, combining two reduced SDDs (e.g., using the conjoin or disjoin operator) may lead to a reduced SDD whose size is exponential in the size of inputs.
Hence, if we activate compression in Algorithm~\ref{alg:apply}, the algorithm may take exponential time in the worst-case.
Second, conditioning a reduced SDD on a literal may exponentially increase its size (assuming the result is also reduced).
Third, forgetting a variable (i.e., existentially quantifying it) from a reduced SDD may exponentially increase its size (again, assuming that the result is also reduced).
These results may seem discouraging. However, we argue next that, in practice, working with reduced SDDs is actually favorable despite the lack of 
polytime guarantees on these transformations. Moreover, we identify a class of vtrees, called {\em bounded vtrees,} which include right-linear ones,
and which induce reduced SDDs with a polytime \apply\ function.

\begin{table}[tb]
\centering
\small
\begin{tabular}{|c|c|c|c|}\hline
{\bf Notation} & {\bf Transformation} & \rot{SDD} & \rot{\shortstack[l]{Reduced~ \\ SDD}} \\
\hline \hline
\cd & polytime conditioning & \y  & \bn \\
\hline
\fo & polytime forgetting & \bn & \bn \\
\hline
\sfo & polytime singleton forgetting & \y & \bn \\
\hline
\(\ac\) & polytime conjunction & \bn & \bn \\
\hline
\(\bac\) & polytime bounded conjunction & \y & \bn \\
\hline
\(\oc\) & polytime disjunction & \bn & \bn \\
\hline
\(\boc\) & polytime bounded disjunction & \y & \bn \\
\hline
\(\nc\) & polytime negation & \y & \y \\
\hline
\end{tabular}
\caption{\y\ means ``satisfies''; \bn\ means ``does not satisfy''. Satisfaction means the existence of a polytime algorithm that implements the
corresponding transformation.
\label{tab:transformations}}
\end{table}

\section{Canonicity or a Polytime \apply?} 

\begin{table}[tb]
\footnotesize
\centering
\begin{tabular}{|c|rrr|}
\hline
Name & Reduced & Reduced & Unreduced \\
          & SDDs+s & SDDs & SDDs \\
\hline
C17 & 99  & 171  & 286 \\
majority & 123  & 193  & 384 \\
b1 & 166  & 250  & 514 \\
cm152a & 149  & 3,139  & 18,400 \\
cm82a & 225  & 363  & 683 \\
cm151a & 614  & 1,319  & 24,360 \\
cm42a & 394  & 823  & 276,437 \\
cm138a & 463  & 890  & 9,201,336 \\
decod & 471  & 810  & 1,212,302 \\
tcon & 596  & 1,327  & 618,947 \\
parity & 549  & 978  & 2,793 \\
cmb & 980  & 2,311  & 81,980 \\
cm163a & 886  & 1,793  & 21,202 \\
pcle & 785  & 1,366  & n/a \\
x2 & 785  & 1,757  & 12,150,626 \\
cm85a & 1,015  & 2,098  & 19,657 \\
cm162a & 907  & 2,050  & 153,228 \\
cm150a & 1,603  & 5,805  & 17,265,164 \\
pcler8 & 1,518  & 4,335  & 15,532,667 \\
cu & 1,466  & 5,789  & n/a \\
pm1 & 1,810  & 3,699  & n/a \\
mux & 1,825  & 6,517  & n/a \\
cc & 1,451  & 6,938  & n/a \\
unreg & 3,056  & 668,531 & n/a \\
ldd & 1,610  & 2,349  & n/a \\
count & 4,168  & 51,639  & n/a \\
comp & 2,212  & 4,500  & 205,105 \\
f51m & 3,290  & 6,049 & n/a\\
my\_adder & 2,793  & 4,408  & 35,754 \\
cht & 4,832  & 13,311  & n/a \\
\hline
\end{tabular}
\caption{LGSynth89 SDD sizes.}
\label{tab:sizes}
\end{table}

\begin{table}[tb]
\footnotesize
\centering
\begin{tabular}{|c|rrr|}
\hline
Name & Reduced & Reduced & Unreduced \\
          & SDDs+s & SDDs & SDDs \\
\hline
C17 & 0.00 & 0.00 & 0.00 \\
majority & 0.00 & 0.00 & 0.00 \\
b1 & 0.00 & 0.00 & 0.00 \\
cm152a &0.01  & 0.01 & 0.02 \\
cm82a & 0.01 & 0.00 & 0.00 \\
cm151a & 0.04 & 0.00 & 0.04 \\
cm42a & 0.03 & 0.00 & 0.10 \\
cm138a &0.02  & 0.01 & 109.05 \\
decod & 0.04 & 0.01 & 1.40 \\
tcon & 0.05 & 0.00 & 0.33 \\
parity & 0.02 & 0.00 & 0.00 \\
cmb &0.12  &  0.02& 0.06 \\
cm163a & 0.06 & 0.00 & 0.02 \\
pcle & 0.07 &0.01  & n/a \\
x2 & 0.08 & 0.02 & 19.87 \\
cm85a & 0.08 & 0.01 & 0.03 \\
cm162a  &0.08  & 0.01 & 0.16 \\
cm150a  &0.16  & 0.06 & 60.37 \\
pcler8  & 0.18 & 0.05 & 33.32 \\
cu  & 0.19 & 0.10 & n/a \\
pm1  & 0.27 & 0.05 & n/a \\
mux  & 0.19 & 0.09 &  n/a\\
cc & 0.22 & 0.04 & n/a \\
unreg  & 0.66 & 263.06 & n/a \\
ldd  & 0.23 & 0.10 & n/a \\
count  &1.05  & 0.24 & n/a \\
comp & 0.24 & 0.01 & 0.22 \\
f51m  & 0.52 & 0.32 &  n/a\\
my\_adder  & 0.24 & 0.02 & 0.04 \\
cht  & 1.24 & 0.36 & n/a \\
\hline
\end{tabular}
\caption{LGSynth89 compilation times, in seconds.}
\label{tab:times}
\end{table}

One has two options when working with SDDs. The first option is to work with unreduced SDDs, which are not canonical, but are supported by a polytime
\apply\ function. The second option is to work with reduced SDDs, which are canonical
but loose the advantage of a polytime \apply\ function. 
The classical reason for seeking canonicity is that
it leads to a very efficient equivalence test, which takes constant time (both reduced and unreduced SDDs support a polytime equivalence test, but
the one known for unreduced SDDs is not a constant time test). The classical reason for seeking a polytime \apply\ function is to enable bottom-up compilation,
that is, compiling a knowledge base (e.g., CNF or DNF) into an SDD by repeated application of the \apply\ function to components of the knowledge base (e.g., clauses or terms).
If our goal is efficient bottom-up compilation, one may expect that unreduced SDDs provide a better alternative. Empirically, however, this turns out to be
false. Our goal in this section is to shed some light on this phenomena through some empirical evidence and then an explanation. 

We used the SDD package provided by the Automated Reasoning Group at UCLA\footnote{Available at \url{http://reasoning.cs.ucla.edu/sdd/}} in our experiments.
The package works with reduced SDDs, but can be adjusted to work with unreduced SDDs as long as dynamic vtree 
search is not invoked.\footnote{Dynamic vtree search requires reduced SDDs as this reduces the search space over SDDs into one over vtrees.}
In our first experiment, we compiled CNFs from the LGSynth89 benchmarks into the following:
\begin{itemize}
  \item[--] Reduced SDDs respecting an arbitrary vtree. Dynamic vtree search is used to minimize the size of the SDD during compilation, starting from a balanced vtree.
  \item[--] Reduced SDDs respecting a fixed balanced vtree.
  \item[--] Unreduced SDDs respecting a fixed balanced vtree.
\end{itemize}

Tables~\ref{tab:sizes} and~\ref{tab:times} show the corresponding sizes and compilation times. According to these results, 
unreduced SDDs end up several orders of magnitude larger than the reduced ones, with or without dynamic vtree search.
For the harder problems, this translates to orders-of-magnitude increase in compilation times. 
Often, we cannot even compile the input without reduction (due to running out of 4GB of memory), even on relatively easy benchmarks.
For the easiest benchmarks, dynamic vtree search is slower due to the overhead, but yields smaller compilations. 
The benefit of vtree search shows only in harder problems (e.g.,~``unreg'').

This experiment clearly shows the advantage of reduced SDDs over unreduced ones, even though the latter supports a polytime \apply\ function while the former does not. This begs an explanation and we provide one next that we back up by additional experimental results.

The benefit of reduced SDDs is canonicity, which plays a critical role in the performance of the \apply\ function. 
Consider in particular Line~\ref{ln:cache} of Algorithm~\ref{alg:apply}. The test \(\cache(\alpha,\beta,\bop) \neq \nil\) checks whether SDDs \(\alpha\)
and \(\beta\) have been previously combined using the Boolean operator \(\bop\). 
Without canonicity, it is possible that we would have combined some \(\alpha^\prime\) and \(\beta^\prime\) using
\(\bop\), where SDD \(\alpha^\prime\) is equivalent to, but distinct from SDD \(\alpha\) (and similarly for \(\beta^\prime\) and \(\beta\)). 
In this case, the cache test would fail, causing \apply\ to recompute the same result again. 
Worse, the SDD returned by \(\apply(\alpha,\beta,\bop)\) may be distinct from the SDD returned by 
\(\apply(\alpha^\prime,\beta^\prime,\bop)\),
even though the two SDDs are equivalent. This redundancy also happens when \(\alpha\) is not equivalent to \(\alpha^\prime\) (and similarly for \(\beta\) and \(\beta^\prime\)),
\(\alpha \circ \beta\) is equivalent to \(\alpha^\prime \circ \beta^\prime\), but the result returned by \(\apply(\alpha,\beta,\bop)\) is distinct from the one returned by \(\apply(\alpha^\prime,\beta^\prime,\bop)\).

Two observations are due here. First, this redundancy is still under control when calling \apply\ only once: \apply\
runs in \(O(n m)\) time, where \(n\) and \(m\) are the sizes of input SDDs. However, this redundancy becomes problematic when calling \apply\ multiple times (as in
bottom-up compilation), in which case quadratic performance is no longer as attractive. For example, if we use \apply\ to combine \(m\) SDDs of size \(n\) each,
all we can say is that the output will be of size \(O(n^m)\).
The second observation is that the previous redundancy will not occur when working with reduced SDDs
due to their canonicity: Two SDDs are equivalent iff they are represented by the same structure in memory.\footnote{This is due to the technique of 
{\em unique nodes} from OBDDs; see $\uniqueD$ in Algorithm~\ref{alg:apply}.}

This analysis points to the following conclusion: While \apply\ has a quadratic complexity on unreduced SDDs, it may have a worse average complexity than \apply\
on reduced SDDs. Our next experiment is indeed directed towards this hypothesis.

For all benchmarks in Table~\ref{tab:sizes} that can be compiled without vtree search, 
we intercept all non-trivial calls to \apply{} (when $|\alpha|\cdot|\beta| > 500$) and report the size of the output $|\alpha \circ \beta|$ divided by $|\alpha|\cdot|\beta|$. 
For unreduced SDDs, we know that $|\alpha \circ \beta| = O(|\alpha|\cdot|\beta|)$ and that these ratios are therefore  bounded above by some constant. For reduced SDDs, 
however, Theorem~\ref{theo:transformations} states that there exists no constant bound.

\begin{figure}[tb]
  \centering
  \subfloat[Reduced SDDs]{
    \label{fig:size-csdd}
    \includegraphics[height=0.145\textwidth,clip=true,angle=0]{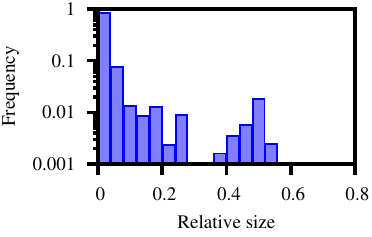}
  }
  \subfloat[Unreduced SDDs]{
    \label{fig:size-sdd} \includegraphics[height=0.145\textwidth,clip=true,angle=0]{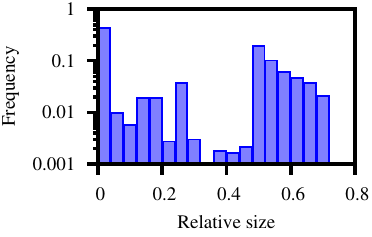}
  }
  \caption{Relative SDD size.}
  \label{fig:sizeratios}
\end{figure}

Figure~\ref{fig:sizeratios} shows the distribution of these ratios for the two methods (note the log scale). 
The number of function calls is 67,809 for reduced SDDs, vs.\ 1,626,591 for unreduced ones.
The average ratio is 0.027 for reduced, vs.\ 0.101 for unreduced.
Contrasting the theoretical bounds, reduced \apply{} incurs much smaller blowups than unreduced \apply{}. 
This is most clear for ratios in the range $[0.48,0.56]$, covering 30\% of the unreduced, but only 2\% of the reduced calls.
  
\begin{figure}[tb]
  \centering
  \subfloat[Reduced SDDs]{
    \label{fig:calls-csdd}
    \includegraphics[height=0.145\textwidth,clip=true,angle=0]{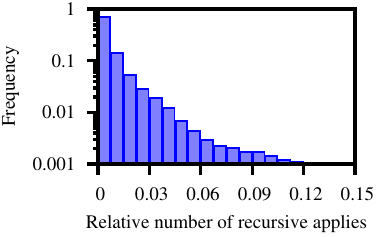}
  }
  \subfloat[Unreduced SDDs]{
    \label{fig:calls-sdd} \includegraphics[height=0.145\textwidth,clip=true,angle=0]{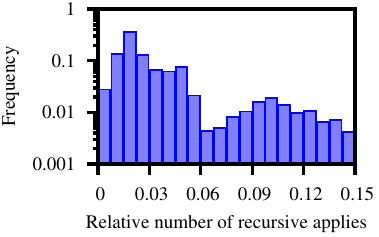}
  }
  \caption{Relative number of recursive \apply{} calls.}
  \label{fig:callsratios}
\end{figure}
The results are similar when looking at runtime for individual \apply{} calls, which we measure by the number of recursive \apply{} calls $r$.
Figure~\ref{fig:callsratios} reports these, again relative to $|\alpha|\cdot|\beta|$.
The ratio $r/(|\alpha|\cdot|\beta|)$ is on average 0.013 for reduced SDDs, vs.\ 0.034 for unreduced ones These results corroborate our earlier analysis, suggesting
that canonicity is quite important for the performance of bottom-up compilers as they make repeated calls to the \apply\ function.
In fact, this can be more important than a polytime \apply, perhaps contrary to common wisdom which seems to emphasize the importance of polytime \apply\
in effective bottom-up compilation (e.g.,~\citet{pipatsrisawat08compilation}).

\section{Bounded Vtrees}

A {\em bounded vtree} is one for which the number of variables in any left subtree is bounded. This includes right-linear vtrees which give rise to OBDDs, since each
left subtree contains a single variable in this case. We now have the following.

\begin{theorem}
The time and space complexity of Algorithm~\ref{alg:apply}, with compression, is in \(O(nm)\), where \(n\) and \(m\) are the sizes of its inputs, assuming that the
input SDDs are reduced and respect a bounded vtree.
\end{theorem}
The compression step of Algorithm~\ref{alg:apply} identifies elements \((p_i,s)\) and \((p_j,s)\) that share sub \(s\), and merges these elements into the element
\((p_i \vee p_j,s)\) by calling \apply\ recursively to disjoin primes \(p_i\) and \(p_j\). Since the vtree is bounded, primes \(p_i\) and \(p_j\) must be over a bounded number
of variables. Hence, the complexity of compression is bounded, leading \apply\ to have the same complexity with or without compression.

 For example, in right-linear vtrees (i.e., OBDDs), primes are literals over a single variable. Hence, all decision nodes are of the form $\{(X,\alpha),(\neg X,\beta)\}$. 
 On these, compression occurs when $\alpha = \beta$, resulting in the partition $\{(X \vee \neg X,\alpha)\} = \{(\top,\alpha)\}$, which trimming replaces by $\alpha$. 
 This corresponds to the OBDD reduction rule that eliminates decision nodes with isomorphic children~\citep{Bryant86}.
 
 \citet{xue12} showed a class of Boolean functions whose OBDDs have exponential size with respect to certain orders (right-linear vtrees), but which have SDDs of linear
 size when the vtrees are not right-linear (but have the same left-to-right variable order). The used vtrees, however, were not bounded. 
 It would be interesting to see if a similar result can be obtained for bounded vtrees.

\section{Conclusions} 
We have shown that the \apply\ function on reduced SDDs can take exponential time in  the worst case, resolving a question that has been open since SDDs were first
introduced. We have also pursued some of the theoretical and practical implications of this result. On the theoretical side, we showed that it implies 
an exponential complexity for various transformations, such as conditioning and existential quantification. We also identified the class of bounded vtrees, for
which the \apply\ function can be implemented in polynomial time even for reduced SDDs. On the practical side, we argued empirically that working with reduced SDDs remains favorable, 
despite the polytime complexity of the \apply\ function on unreduced SDDs. The canonicity of reduced SDDs, we argued, is 
more valuable for bottom-up compilation than a polytime \apply\ due to its role in facilitating caching and dynamic vtree search. Our findings appear contrary to
some of the common wisdom on the relationship between bottom-up compilation, canonicity and the complexity of the \apply\ function.

\section{ Acknowledgments}

We thank Arthur Choi, Doga Kisa, and Umut Oztok for helpful suggestions.
This work was supported by ONR grant
\#N00014-12-1-0423, NSF grant \#IIS-1118122,
NSF grant \#IIS-0916161, and the Research Foundation-Flanders (FWO-Vlaanderen).
GVdB is also at KU~Leuven, Belgium.


\appendix
\section{Complexity of Transformations}  \label{app:sdd trans}
We now prove Theorem~\ref{theo:transformations}, stating that the results in Table~\ref{tab:transformations} hold.
We will first show the results for unreduced SDDs, and then prove the results for reduced SDDs.

For \emph{unreduced SDDs}, \citet{Darwiche11} showed support for $\bac$, $\boc$, and $\nc$ (see Algorithm~\ref{alg:apply}).
We show support for unreduced $\cd$ next.
\begin{theorem}
  We can condition an unreduced SDD on a literal $\ell$ in polynomial time by replacing $\ell$ by $\top$ and $\neg \ell$ by $\bot$. Moreover, when removing all elements whose prime is equivalent to $\bot$, the resulting sentence is an unreduced SDD.
\end{theorem}
\begin{proof}
It is clear that the procedure transforms $\alpha$ into a sentence that is logically equivalent to $\alpha|\ell$:
the first step directly follows the definition of conditioning, and the second step maintains logical equivalence.
We need to show next that the result is syntactically an SDD, by showing that the primes in its partitions are consistent, exhaustive, and mutually exclusive. 
The second step enforces consistency of the primes. Moreover, if the primes are exhaustive, that is, $p_1 \lor \dots \lor p_n \equiv \top$, then $p_1|\ell \lor \dots \lor p_n|\ell \equiv (p_1 \lor \dots \lor p_n)|\ell \equiv \top|\ell \equiv \top$, and the result of conditioning is also exhaustive.
Finally, when $p_i$ and $p_j$ are mutually exclusive, that is, $p_i \land p_j \equiv \bot$, then $p_i|\ell \land p_j|\ell \equiv (p_i \land p_j)|\ell \equiv \bot|\ell \equiv \bot$, and the conditioned primes are also mutually exclusive.
\end{proof}
Support for $\sfo$ follows from the support for $\cd$ and $\boc$.
The negative results for $\fo$, $\ac$ and $\oc$ follow from identical OBDD results in \citet{darwicheJAIR02}, 
and the fact that OBDDs are a special case of~SDDs.

For \emph{reduced SDDs}, the negative $\fo$, $\ac$ and $\oc$ results also follow from OBDD results.
It is also clear from Algorithm~\ref{alg:apply} that negating a reduced SDD $\alpha$ by computing $\apply(\alpha,\top,\text{xor})$ does not cause any subs to become equivalent. Therefore, negating a reduced SDD leads to a reduced result, and reduced SDDs support $\nc$.
The remaining results in Table~\ref{tab:transformations}, on $\cd$, $\sfo$, $\bac$ and $\boc$ are discussed next.

\begin{theorem} \label{thm:cond}
There exists a class of Boolean functions \(f(X_1, \ldots, X_n)\) and vtrees \(T_n\) for which the reduced SDD has size \(O(n)\), yet the
reduced SDD for the function \(f(X_1,\ldots,X_n) | \ell\) has size \(\Omega(2^n)\) for some literal \(\ell\).
\end{theorem} 

\begin{proof}
Consider the function
\begin{align*}
 &f^c_n(\X,\Y,\Z,W) = \\
 &\qquad\bigvee_{i=1}^n \bigwedge_{j=1}^{i-1} \neg Y_j  \land Y_i \land \left[ \left( X_i \land \left(W \lor Z_i \right) \right) \lor \left(\neg X_i  \land Z_i \right) \right]
\end{align*}
and the vtree depicted in Figure~\ref{fig:vtree:c}.
  \begin{figure}
    \centering
    \subfloat[For conditioning]{
      \label{fig:vtree:c}
      \begin{tikzpicture}[vtree]
      \node{\Tree [.1 [.2 $\X$ $\Y$ ] $\Z \cup \{W\}$ ]};
      \end{tikzpicture}
    }
    ~
    \subfloat[For forgetting]{
      \label{fig:vtree:forget}
      ~~~~~\begin{tikzpicture}[vtree]
      \node{\Tree [.1 
          $L$
          $\X$
        ]
        };
      \end{tikzpicture}~~~~~    
    }
    \caption{Vtree structures used in proofs}
    \label{fig:vtree2}
  \end{figure}
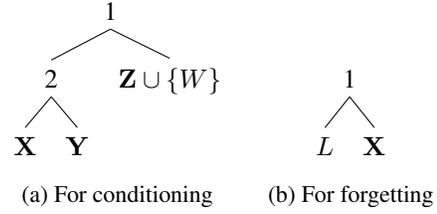

The root of the reduce SDD for $f^c_n$ is an ($\X\Y,\Z{}W$)-partition that respects vtree node 1, consisting of elements
\begin{align*}
\bigcup_{i=1}^n & \left\{ \left( \bigwedge_{j=1}^{i-1} \neg Y_j \land Y_i \land X_i,~~W \lor Z_i \right), \right.\\
&~~~ \left. \left( \bigwedge_{j=1}^{i-1} \neg Y_j \land Y_i \land \neg X_i,~~Z_i \right) \right\},
\end{align*}
together with $\left( \bigwedge_{j=1}^{n} \neg Y_j,~\false \right)$ to make it exhaustive.
The size of this partition is linear in $n$. It has the same primes as the unreduced SDD for $f^a_n$ used in the proof of Theorem~\ref{theo:compression}, only now the partition is compressed, as all subs are distinct.

The primes of this partition can be represented as compressed \xypart{}s, exactly as in the second step for Theorem~\ref{theo:compression}. The remaining primes and subs (over $\X$, over $\Y$, and over $\Z \cup \{W\}$) are all simple conjunctions or disjunctions of literals that have a linear reduced SDD representation for any vtree.

We have now obtained a polysize SDD. However, when we condition this SDD on the literal $W$, all $n$ subs of the form $W \lor Z_i$ become equivalent to $\top$. Their elements need to be compressed into the single element $\left( \bigvee_{i=1}^n\bigwedge_{j=1}^{i-1} \neg Y_j \land Y_i \land X_i,~\top \right)$. Its prime is again the function $f^b_n(\X,\Y)$ for which Lemma~\ref{thm:expsddx} states that the reduced SDD wrt vtree node 2 has exponential size. 
\end{proof}

\begin{theorem} \label{thm:conj}
There is a class of Boolean functions \(f(X_1, \ldots, X_n)\) and vtrees \(T_n\) for which the reduced SDD has size \(O(n)\), yet
the reduced SDD for the Boolean function \(f(X_1, \ldots, X_n) \wedge \ell\) has size \(\Omega(2^n)\) for some literal \(\ell\). 
\end{theorem}
\begin{proof}
  Consider again the reduced SDD for $f^c_n$ that was constructed in the proof of Theorem~\ref{thm:cond} for the vtree in Figure~\ref{fig:vtree:c}.
  Conjoining this SDD with the SDD for literal $W$ makes the $n$ subs of the form $W \lor Z_i$ equivalent to $W \land (W \lor Z_i) = W$. 
  Compressing these creates the element $\left( \bigvee_{i=1}^n\bigwedge_{j=1}^{i-1} \neg Y_j \land Y_i \land X_i,~W \right)$, whose prime is again $f^b_n(\X,\Y)$, which has no polysize reduced SDD for vtree node~2.
\end{proof}

This already proves that \apply{} is worst-case exponential when performing conjunctions on reduced SDDs.
Given that reduced SDDs support polytime negation, this result generalizes to any binary Boolean operator $\circ$ that is functionally complete together with negation~\citep{wernick1942complete}. Support for these operators would allow us to do polytime conjunction by combining $\circ$ and negation. One such operator is disjunction, which is therefore also is worst-case exponential.

Suppose now that we can perform singleton forgetting in polytime, which is defined as $\exists L.\alpha = (\alpha|L) \lor (\alpha|\neg L)$. 
Then given any two reduced SDDs $\beta$ and $\gamma$ respecting the same vtree $T$, we can obtain $\beta \lor \gamma$ in polytime 
as follows.  Add a new variable $L$ to vtree $T$, as depicted in Figure~\ref{fig:vtree:forget}.
The reduced SDD $\alpha$ for the function $(L \land \beta) \lor (\neg L \land \gamma)$ has the root partition $\{(L, \beta),(\neg L, \gamma)\}$.
Forgetting $L$ from $\alpha$ results in the reduced SDD for $\beta \lor \gamma$. Hence, if single forgetting can be done in polytime, then bounded disjunction can
also be done in polytime. Since the latter is impossible, the former is also impossible.

\bibliographystyle{aaai}
\bibliography{references}

\end{document}